\documentclass{article}

\usepackage[final]{iai_neurips_2024}
\usepackage{makecell}
\newtheorem{definition}{Definition}
\newtheorem{proposition}{Proposition}
\usepackage{amsmath}
\usepackage{algorithm}
\newcommand{\isometrypursuit}{{\sc IsometryPursuit}}
\newcommand{\brute}{{\sc BruteSearch}}
\newcommand{\greedy}{{\sc GreedySearch}}
\newcommand{\tsip}{{\sc TwoStageIsometryPursuit}}
\usepackage[font=small]{caption}
\usepackage{algorithmic}
\usepackage{subcaption}
\usepackage{graphicx}
\usepackage[utf8]{inputenc} % allow utf-8 input
\usepackage[T1]{fontenc}    % use 8-bit T1 fonts
\usepackage{hyperref}       % hyperlinks
\usepackage{url}            % simple URL typesetting
\usepackage{booktabs}       % professional-quality tables
\usepackage{amsfonts}       % blackboard math symbols
\usepackage{nicefrac}       % compact symbols for 1/2, etc.
\usepackage{microtype}      % microtypography
\usepackage{xcolor}         % colors
\usepackage{tikz-cd}
\newenvironment{proof}{\paragraph{Proof:}}{\hfill$\square$}
\newenvironment{acknowledgments}{\section*{Acknowledgments}}{}
\newcommand{\M}{\mathcal{M}}
\newcommand{\N}{\mathcal{N}}

\title{Isometry pursuit}

\author{%
  Samson Koelle \\
  Amazon  \\
  koelle@amazon.com
  \And
  Marina Meila \\
  Department of Statistics\\
  University of Washington \\
  mmp@uw.edu
}

\begin{document}

\maketitle

\begin{abstract}
Isometry pursuit is a convex algorithm for identifying orthonormal column-submatrices of wide matrices.
It consists of a novel normalization method followed by multitask basis pursuit.
Applied to Jacobians of putative coordinate functions, it helps identity isometric embeddings from within interpretable dictionaries.
We provide theoretical and experimental results justifying this method.
For problems involving coordinate selection and diversification, it offers a synergistic alternative to greedy and brute force search.
\end{abstract}

\footnotetext[1]{Work conducted outside of Amazon.}

\section{Introduction}
\label{sec:introduction}

Many real-world problems may be abstracted as selecting a subset of the columns of a matrix representing stochastic observations or analytically exact data.
This paper focuses on a simple such problem that appears in interpretable learning and diversification.
Given a rank $D$ matrix $ X \in \mathbb R^{D \times P}$ with $P > D$, select a square submatrix $ X_{. S}$ where subset $ S \subset P$ satisfies $| S| = D$ that is as orthonormal as possible.

This problem arises in interpretable learning specifically because while the coordinate functions of a given feature space may have no intrinsic meaning, it is sometimes possible to generate a dictionary of interpretable features which may be considered as potential parametrizing coordinates.
When this is the case, selection of candidate interpretable features as coordinates can take the above form.
While implementations vary across data and algorithmic domains, identification of such coordinates generally aids mechanistic understanding, generative control, and statistical efficiency.

This paper shows that an adapted version of the algorithm in \citet{Koelle2024-no} leads to a convex procedure that can improve upon greedy approaches such as those in \citet{5895106, NEURIPS2019_6a10bbd4, Kohli2021-lr, Jones2007-uc} for finding isometries.
The insight leading to isometry pursuit is that multitask basis pursuit applied to an appropriately normalized $ X$ selects orthonormal submatrices.
Given vectors in $\mathbb R^D$, the normalization log-symmetrizes length and favors those closer to unit length, while basis pursuit favors those which are orthogonal.
Our results formalize this intuition within a limited setting, and show the usefulness of isometry pursuit as a trimming procedure prior to brute force search for diversification and interpretable coordinate selection.
We also introduce a novel ground truth objective function against which we measure the success of our algorithm, and discuss the reasonableness of the trimming procedure.

 \footnotetext[2]{Code is available at \url{https://github.com/sjkoelle/isometry-pursuit}.}
\section{Background}

Our algorithm is motivated by spectral and convex analysis.

\subsection{Problem}

Our goal is, given a matrix $ X \in \mathbb R^{D \times P}$, to select a subset $ S \subset [P]$ with $| S| = D$ such that $X_{.  S}$ is as orthonormal as possible in a computationally efficient way.
To this end, we define a ground truth loss function that measures orthonormalness, and then introduce a surrogate loss function that convexifies the problem so that it may be efficiently solved.

\subsection{Interpretability and isometry}

Our motivating example is the selection of data representations from within sets of putative coordinates: the columns of a provided wide matrix.
Compared with Sparse PCA \citep{Dey2017-mx, Bertsimas2022-qo, Bertsimas2022-dv}, we seek a low-dimensional representation from the set of these column vectors rather than their span.

This method applies to interpretability, for which parsimony is at a premium.
Interpretability arises through comparison of data with what is known to be important in the domain of the problem.
This knowledge often takes the form of a functional dictionary.
Evaluation of independence of dictionary features arises in numerous scenarios \citep{Chen2019-km, Koelle2022-ju, He2023-ch}.
The requirement that dictionary features be full rank has been called functional independence \citep{Koelle2022-ju} or feature decomposability \citep{templeton2024scaling}, with connection between dictionary rank and independence via the implicit function theorem.
Besides independence, the metric properties of such dictionary elements are of natural interest.
This is formalized through the notion of differential.

\begin{definition}
The \textbf{differential} of a smooth map $\phi:\mathcal M \to \mathcal N$ between $D$ dimensional manifolds $\M \subseteq \mathbb R^B$ and $\N \subseteq \mathbb R^P$ is a map in tangent bases $x_1 \dots x_{D}$ of $T_\xi \M$ and $y_1 \dots y_{D}$ of $T_{\phi(\xi)} \N$ consisting of entries
\begin{align}
\label{eq:diff}
    D\phi (\xi) = \begin{bmatrix}
    \frac{\partial \phi_1  }{\partial x_1}(\xi)  & \dots & \frac{\partial \phi_1 }{\partial x_D}(\xi)  \\
    \vdots & & \vdots \\
    \frac{\partial \phi_D }{\partial x_1}(\xi)  & \dots & \frac{\partial \phi_{D}  }{\partial x_{D}}(\xi) 
    \end{bmatrix}.
\end{align}
\end{definition}

It is not always necessary to explicitly estimate tangent spaces when applying this definition.
The most commonly encountered manifolds are vector spaces for which the tangent spaces are trivial.
This is the case for full-rank tabular data, for which isometry has a natural interpretation as a type of diversification, and often for the latent spaces of deep learning models.
In this case, $B = D$.

\begin{definition}
\label{def:isometric_at_a_point}
A map $\phi$ between $D$ dimensional submanifolds with inherited Euclidean metric $\mathcal M \subseteq R^{B}$ and $\mathcal N  \subseteq R^{P}$
$\phi$ is an \textbf{isometry at a point} $\xi \in \mathcal M$ if
\begin{align}
{D \phi (\xi)}^T D \phi (\xi) = I_D.
\end{align}
That is, $\phi$ is an isometry at $\xi$ if $D \phi (\xi)$ is orthonormal.
\end{definition}

The applications of pointwise isometry are themselves manifold.
Pointwise isometric embeddings faithfully preserve high-dimensional geometry.
For example, Local Tangent Space Alignment \citep{ZhangZ:04}, Multidimensional Scaling \citep{ChenBuja:localMDS09} and Isomap \citep{tenenbaum2000ggf} non-parametrically estimate embeddings that are as isometric as possible.
Another approach stitches together pointwise isometries selected from a dictionary to form global embeddings \citep{Kohli2021-lr}.
The method is particularly relevant since it constructs such isometries through greedy search, with putative dictionary features added one at a time.

That $D\phi$ is orthonormal has several equivalent formulations.
The one motivating our ground truth loss function comes from spectral analysis.
\begin{proposition}
\label{prop:orthonormal_spectrum}
The singular values $\sigma_1 \dots \sigma_D$ are equal to $1$ if and only if $U \in \mathbb{R}^{D \times D}$ is orthonormal.
\end{proposition}
On the other hand, the formulation that motivates our convex approach is that orthonormal matrices consist of $D$ coordinate features whose gradients are orthogonal and of unit length.
\begin{proposition}
\label{prop:orthonormal_basis}
The component vectors $u_1 \dots u_D \in \mathbb R^B$ form a orthonormal matrix if and only if, for all $d_1, d_2 \in [D], \langle u_{d_1}, u_{d_2} \rangle = \begin{cases}
1 \; d_1 = d_2 \\ 
0 \; d_1 \neq d_2 
\end{cases}$.
\end{proposition}

\subsection{Subset selection}

Given a matrix $ X \in \mathbb R^{D \times P}$, we compare algorithmic paradigms for solving problems of the form
\begin{align}
\label{prog:ground_truth}
\arg \min_{ S \in \binom{[P]}{D}} l ( X_{. S})
\end{align}
where $\binom{[P]}{D} = \left\{ A \subseteq [P] : \left|A\right| = D \right\}$.
Brute force algorithms consider all possible solutions.
These algorithms are conceptually simple, but have the often prohibitive time complexity $O(C_lP^D)$ where $C_l$ is the cost of evaluating $l$.
Greedy algorithms consist of iteratively adding one element at a time to $ S$.
This algorithms have time complexity $O(C_lPD)$ and so are computationally more efficient than brute force algorithms, but can get stuck in local minima.
Formal definitions are given in Section \ref{sec:algorithms}.

Sometimes, it is possible to introduce an objective which convexifies problems of the above form.
Solutions
\begin{align}
\arg \min f(\beta) : Y  = X\beta 
\end{align}
to the overcomplete regression problem $Y = X \beta$ are a classic example \citep{Chen2001-hh}.
When $f(\beta) = \|\beta\|_0$, this problem is non-convex, and is thus suitable for greedy or brute algorithms, but when $f(\beta) =\|\beta\|_1$, the problem is convex, and may be solved efficiently via interior-point methods.
When the equality constraint is relaxed, Lagrangian duality may be used to reformulate as a so-called Lasso problem, which leads to an even richer set of optimization algorithms. % cite FISTA< glmnet, coordinate descent

The form of basis pursuit that we apply is inspired by the group basis pursuit approach in \citet{Koelle2022-ju}.
In group basis pursuit (which we call multitask basis pursuit when grouping is dependent only on the structure of matrix-valued response variable $y$) the objective function is $f(\beta) = \|\beta\|_{1,2} := \sum_{p=1}^P \|\beta_{p.}\|_2$  \citep{Yuan2006-bt, Obozinski2006-kq, Yeung2011-fg}.
This objective creates joint sparsity across entire rows $\beta_{p.}$ and was used in \citet{Koelle2022-ju} to select between sets of interpretable features.
\section{Method}

We adapt the group lasso paradigm used to select independent dictionary elements in \citet{Koelle2022-ju, Koelle2024-no} to select pointwise isometries from a dictionary.
We first define a ground truth objective computable via brute and greedy algorithms that is uniquely minimized by orthonormal matrices.
We then define the combination of normalization and multitask basis pursuit that approximates this ground truth loss function.
We finally give a brute post-processing method for ensuring that the solution is $D$ sparse.

\subsection{Ground truth}
\label{sec:ground_truth}

We'd like a ground truth objective to be minimized uniquely by orthonormal matrices, invariant under rotation, and depend on all changes in the matrix.
Deformation \citep{Kohli2021-lr} and nuclear norm \citep{Boyd2004-ql} use only a subset of the differential's information and are not uniquely minimized at unitarity, respectively.
We therefore introduce an alternative ground truth objective that satisfies the above desiderata and has convenient connections to isometry pursuit.

This objective is
\begin{align}
l_{c}: \mathbb R^{D \times P} &\to \mathbb R^{+} \\
X &\mapsto \sum_{d = 1}^D g(\sigma_d( X), c)
\end{align}
where $\sigma_d ( X)$ is the $d$-th singular value of $ X$ and
\begin{align}
g: \mathbb R^+ \times \mathbb R^+ &\to \mathbb R^+ \\
t,c &\mapsto \frac{e^{t^c} + e^{t^{-c}}}{2e}.
\end{align}
Using Proposition \ref{prop:orthonormal_spectrum}, we can check that $l_c$ is uniquely maximized by orthonormal matrices.
Moreover, $g$ is convex, and $l_c( X^{-1}) = l_c( X)$ when $X$ is invertible.
Figure \ref{fig:losses} gives a graph of $l_c$ when $D=1$ and compares it with that produced by basis pursuit after normalization as in Section \ref{sec:normalization}.
% NOTE (Sam): Do we need proofs of maximized by orthonormal matrices and convex?
% Can we prove this is a norm?

Our ground truth program is therefore 
\begin{align}
\label{prog:ground_truth}
\arg \min_{ S \in \binom{[P]}{d}} l_c ( X_{. S}).
\end{align}
Regardless of the convexity of $l_c$, brute combinatorial search over $[P]$ is inherently non-convex.

\subsection{Normalization}
\label{sec:normalization}

Since basis pursuit methods tend to select longer vectors, selection of orthonormal submatrices requires normalization such that both long and short candidate basis vectors are penalized in the subsequent regression.
We introduce the following definition.

\begin{definition}[Symmetric normalization]
A function $q: \mathbb R^D \to \mathbb R^+ $ is a symmetric normalization if 
\begin{align}
\arg \max_{v \in \mathbb R^D} \ q (v) &=\{ v : \|v\|_2 = 1 \} \\
q(v) &= q(\frac{v}{\|v\|_2^2}) \\
q(v_1) &= q(v_2) \; \forall \; v_1, v_2 \in \mathbb R^D : \|v_1\|_2 = \|v_2\|_2.
\end{align} \label{def:symmetric_normalization}
\end{definition}

We use such functions to normalize vector length in such a way that vectors of length $1$ prior to normalization have longest length after normalization and vectors are shrunk proportionately to their deviation from $1$. 
That is, we normalize vectors by 
\begin{align}
n: \mathbb R^D  &\to \mathbb R^D \\
v &\mapsto {q(v) }v
\end{align}
and matrices by
\begin{align}
w: \mathbb R^{D \times P}  &\to \mathbb R^D \\
 X_{.p} &\mapsto n( X_{.p}) \; \forall \; p \in [P].
\end{align}

In particular, given $c > 0$, we choose $q$ as follows.
\begin{align}
q_c: \mathbb R^D  &\to \mathbb R^+ \\
v  &\mapsto \frac{e^{\|v\|_2^c} + e^{\|v\|_2^{-c}}}{2e}.
\label{eq:normalization}
\end{align}

Besides satisfying the conditions in Definition \ref{def:symmetric_normalization}, this normalization has some additional nice properties.
First, $q$ is convex.
Second, it grows asymptotically log-linearly.
Third, while $\exp(-|\log t|) = \exp(-\max (t, 1/t))$ is a seemingly natural choice for normalization, it is non smooth, and the LogSumExp \citep{Boyd2004-ql} replacement of $\max (t, 1/t)$ with $ \log (\exp (t ) + \exp(1/t))$ simplifies to \ref{eq:normalization} upon exponentiation.
Finally, the parameter $c$ grants control over the width of the basin, which may be useful for avoiding numerical issues arising close to $0$ and $\infty$.

\begin{figure}
\centering
\subcaptionbox{Ground truth loss \label{cat}}
{\includegraphics[width = .32\textwidth]{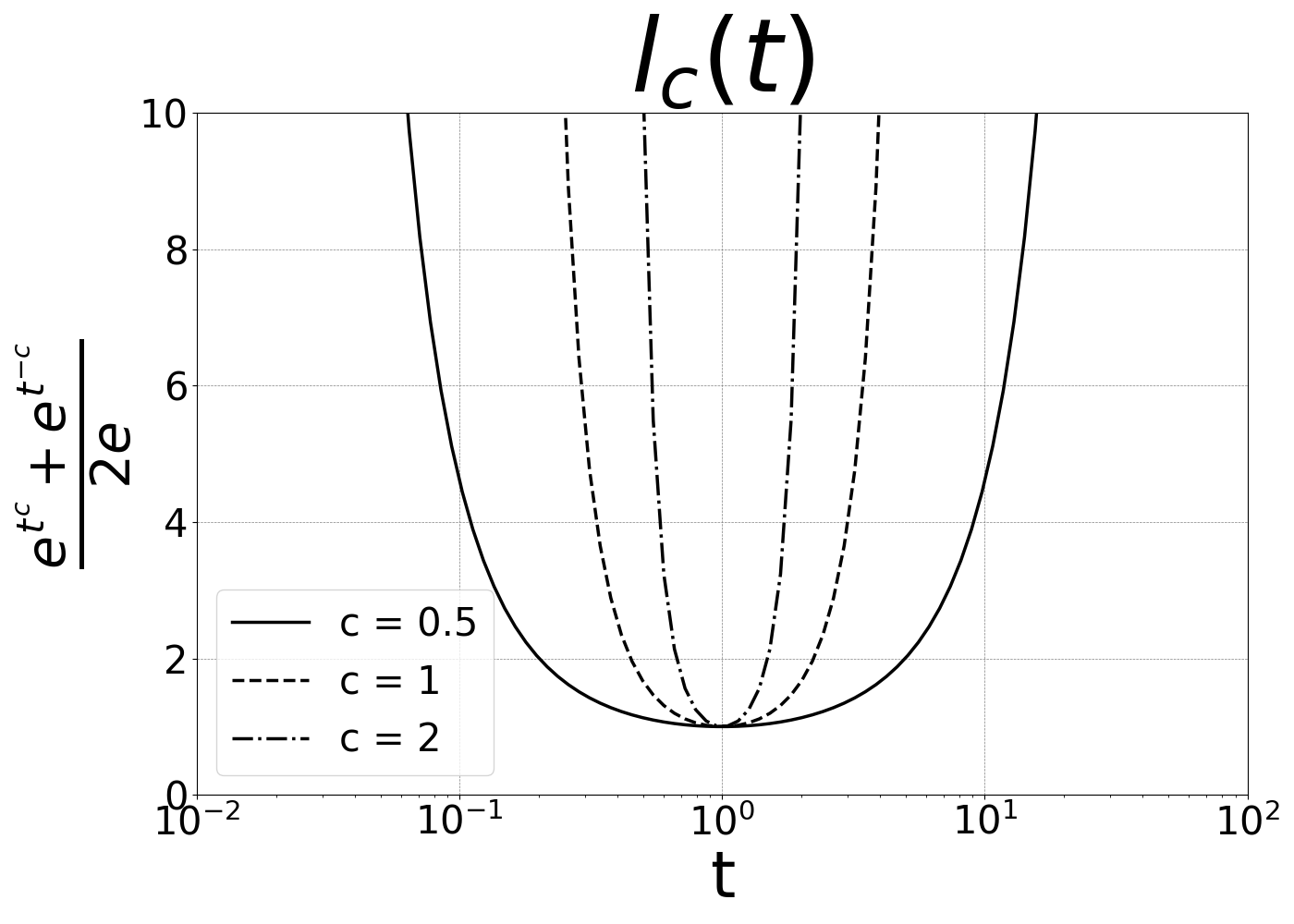}}
\label{fig:gt_loss}
\subcaptionbox{Normalized length as a function of unnormalized length \label{elephant}}
{\includegraphics[width = .32\textwidth]{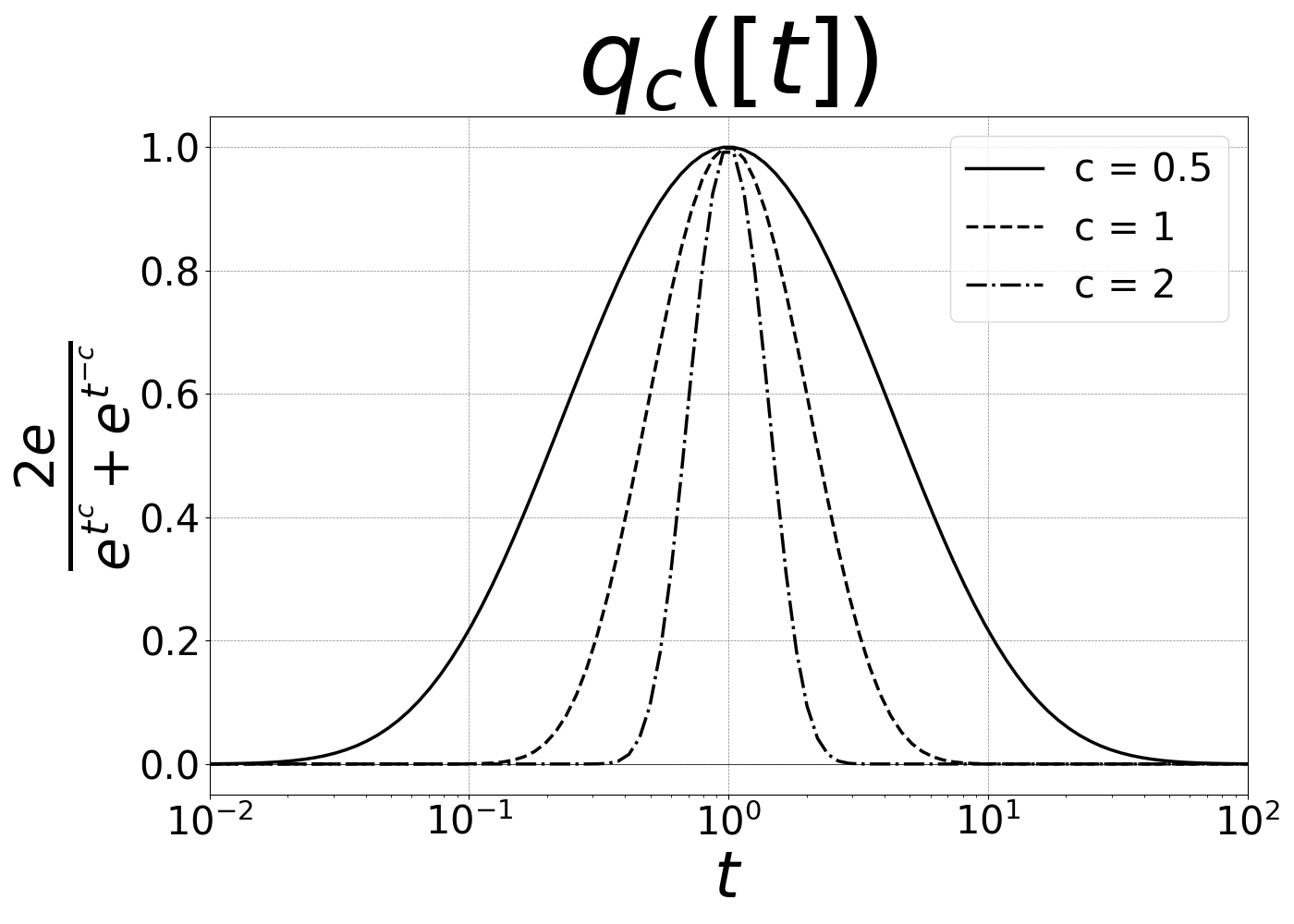}}
\subcaptionbox{Basis pursuit loss \label{snootfellow}}
{\includegraphics[width = .32\textwidth]{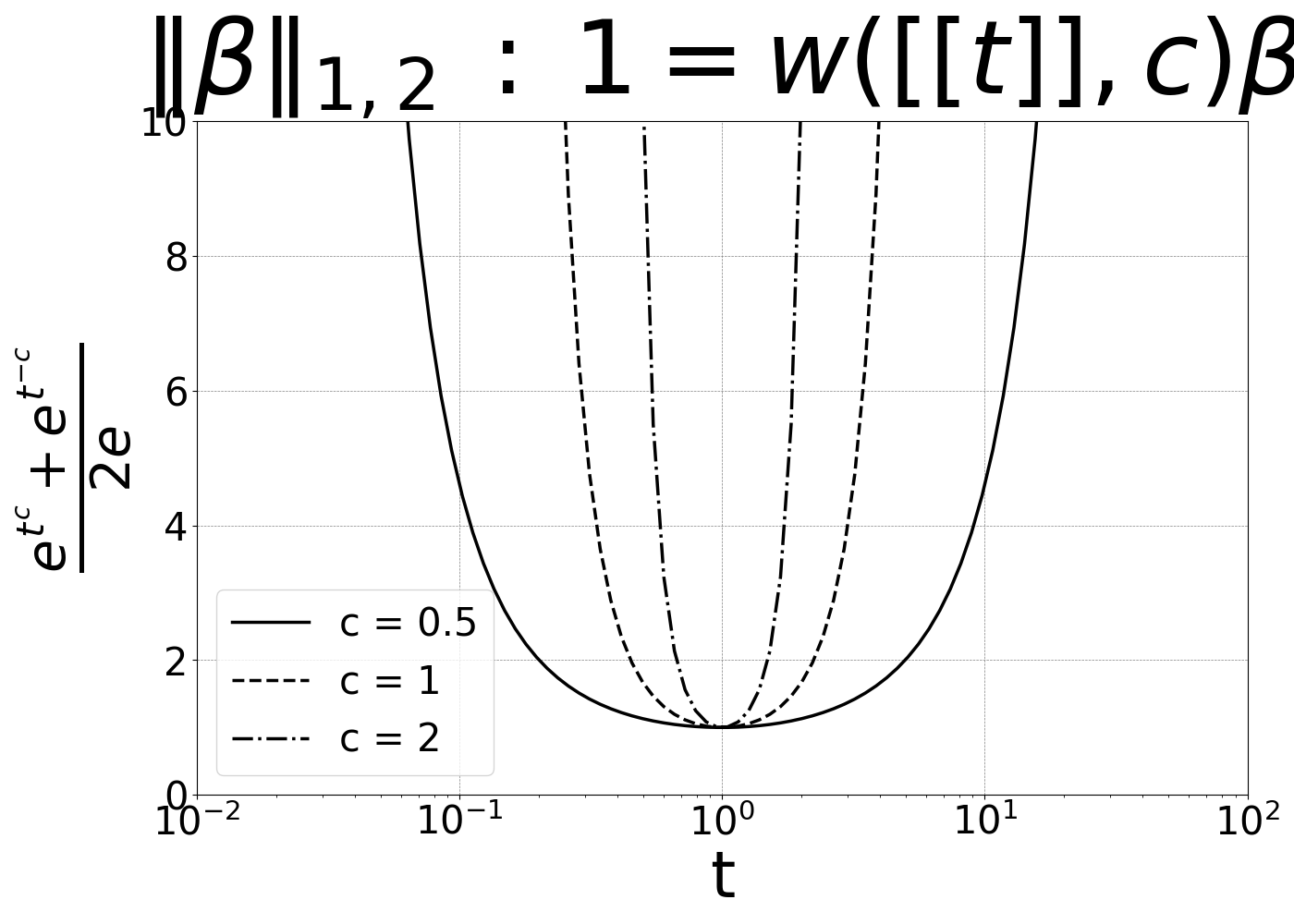}}
\caption{Plots of ground truth loss, normalized length, and basis pursuit loss for different values of $c$ in the one-dimensional case $D = 1$.
The two losses are equivalent in the one-dimensional case.}
\label{fig:losses}
\end{figure}

\subsection{Isometry pursuit}

Isometry pursuit is the application of multitask basis pursuit to the normalized design matrix $w(X, c)$ to identify submatrices of $ X$ that are as orthonormal as possible.
Define the multitask basis pursuit penalty 
\begin{align}
\label{eq:bp}
\| \cdot \|_{1,2}: \mathbb R^{P \times D} &\to \mathbb R^+ \\ 
\beta &\mapsto  \sum_{p=1}^P  \|\beta_{p.}\|_2.
\end{align}
Given a matrix $Y \in \mathbb R^{D \times D}$, the multitask basis pursuit solution is
\begin{align}
\label{prog:multitask_basis_pursuit}
\widehat \beta_{MBP} (X, Y)  := \arg \min_{\beta \in \mathbb R^{P \times D}} \| \beta \|_{1,2} \; : \;Y =  X \beta.
\end{align}
Isometry pursuit is then given by
\begin{align}
\label{prog:isometry_pursuit}
\widehat \beta_c ( X) := \widehat \beta_{MBP} ( w(X,c), I_D )
\end{align}
where $I_D$ is the $D$ dimensional identity matrix and recovered functions are the indices of the dictionary elements with non-zero coefficients.
That is, they are given by $S(\beta)$ where
\begin{align}
S: \mathbb{R}^{P \times D} &\to \binom{[P]}{D} \\
\beta &\mapsto \left\{ p \in [P] :  \|\beta_{p.}\| > 0 \right\}.
\end{align}
\begin{algorithm}[H]
\caption{\isometrypursuit(Matrix ${X} \in \mathbb{R}^{D \times P}$, scaling constant $c$)}
\begin{algorithmic}[1]
\STATE {Normalize} $X_c = w({X},c)$
\STATE {Optimize} $\widehat \beta = \widehat \beta_{MBP} (X_c, I_D)$
\STATE {\bf Output} $\widehat{S} = S (\widehat \beta)$
\end{algorithmic}
\end{algorithm}

\subsection{Theory}

The intuition behind our application of multitask basis pursuit is that submatrices consisting of vectors which are closer to 1 in length and more orthogonal will have smaller loss.
A key initial theoretical assertion is that  \isometrypursuit~ is invariant to choice of basis for $ X$.
\begin{proposition}
\label{prop:basis_pursuit_selection_invariance}
Let $U \in \mathbb R^{D \times D}$ be orthonormal.
Then $S(\widehat \beta  (U  X)) = S(\widehat \beta ( X))$.
\end{proposition}
A proof is given in Section \ref{proof:basis_pursuit_program_invariance}.
This has as an immediate corollary that we may replace $I_D$ in the constraint by any orthonormal $D \times D$ matrix.

We also claim that the conditions of the consequent of Proposition \ref{prop:orthonormal_basis} are satisfied by minimizers of the multitask basis pursuit objective applied to suitably normalized matrices in the special case where a rank $D$ orthonormal submatrix exists and $|S| = D$.
\begin{proposition}
Let $w_c$ be a normalization satisfying the conditions in Definition \ref{def:symmetric_normalization}.
Then $\arg \min_{X_{.S} \in \mathbb R^{D \times D}} \widehat \beta_c ( X_{.S}) $ is orthonormal and, given $X$ is orthonormal, $ \| \beta \|_{1,2} \; : \; I_D = w ({  X}, c) \beta = D$.
\label{prop:unitary_selection}
\end{proposition}
While this Proposition falls short of showing that an orthonormal submatrix will be selected should one be present, it provides intuition justifying the preferential efficacy of \isometrypursuit~ on real data.
A proof is given in Section \ref{sec:local_isometry_proof}.

\subsection{Two-stage isometry pursuit}

Since we cannot ensure either that $|\widehat {  S}| = D$ or that a orthonormal submatrix $X_{.S}$ exists, we first use the convex problem to prune and then apply brute search upon the substantially reduced feature set.

\begin{algorithm}[H]
\caption{\tsip(Matrix ${X} \in \mathbb{R}^{D \times P}$, scaling constant $c$)}
\begin{algorithmic}[1]
\STATE $\widehat{S}_{IP} = \text{\isometrypursuit}( X, c)$
\STATE $\widehat{S} = \text {\brute}({X}_{.\widehat{S}_{IP}}, l_c)$
\STATE {\bf Output} $\widehat{S}$
\end{algorithmic}
\end{algorithm}

Similar two-stage approaches are standard in the Lasso literature \cite{Hesterberg2008-iy}.
This method forms our practical isometry estimator, and is discussed further in Sections \ref{sec:discussion} and \ref{sec:deep_dive}.
\section{Experiments}
\label{sec:experiments}

Say you are hosting an elegant dinner party, and wish to select a balanced set of wines for drinking and flowers for decoration.
We demonstrate \tsip~ and \greedy~ on the Iris and Wine datasets \citep{misc_iris_53, misc_wine_109, scikit-learn}.
This has an intuitive interpretation as selecting diverse elements that reflects the peculiar structure of the diversification problem.
Features like \textit{ petal width} are rows in $X$.
They are features on the basis of which we may select among the flowers those which are most distinct from another.
Thus, in diversification, $P = n$.

We also analyze the Ethanol dataset from \citet{Chmiela2018-at, Koelle2022-ju}, but rather than selecting between bourbon and scotch we evaluate a dictionary of interpretable features  - bond torsions - for their ability to parameterize the molecular configuration space.
In this interpretability use case, columns denote gradients of informative features.
We compute Jacoban matrices of putative parametrization functions and project them onto estimated tangent spaces (see \citet{Koelle2022-ju} for preprocessing details).
Rather than selecting between data points, we are selecting between functions which parameterize the data.

For basis pursuit, we use the SCS interior point solver \citep{ocpb:16} from CVXPY \citep{diamond2016cvxpy, agrawal2018rewriting}, which is able to push sparse values arbitrarily close to 0 \citep{cvxpy_sparse_solution}.
Statistical replicas for Wine and Iris are created by resampling across $[P]$.
Due to differences in scales between rows, these are first standardized.
For the Wine dataset, even \brute~ on $\widehat {S}_{IP}$ is prohibitive in $D=13$, and so we truncate our inputs to $D=6$.
For Ethanol, replicas are created by sampling from data points and their corresponding tangent spaces are estimated in $B = 252$.

Figure \ref{fig:isometry_losses} and Table \ref{tab:experiments} show that the $l_1$ accrued by the subset $\widehat S_{G}$ estimated using \greedy~ with objective $l_1$ is higher than that for the subset estimated by \tsip.
This effect is statistically significant, but varies across datapoints and datasets.
Figure \ref{fig:support_cardinalities} details intermediate support recovery cardinalities from \isometrypursuit.
We also evaluated second stage \brute~ selection after random selection of $\widehat S_{IP}$ but do not report it since it often lead to catastrophic failure to satisfy the basis pursuit constraint.
Wall-clock runtimes are given in Section \ref{sec:timing}.

\begin{table}[h!]
\tiny
\centering
\begin{tabular}{|c|c|c|c|c|c|c|c|c|c|c|}
\toprule
Name & $D$ & $P$ & $R$ & $c$ & $l_1(X_{.\widehat{S}_{G}})$ & $|\widehat{S}_{IP}|$ & $l_1(X_{.\widehat{S}})$ & $\thead{\tiny P_R (l_1(X_{.\widehat{S}_{G}})  \\ > l_1(X_{.\widehat{S}_{}}))}$ & $ \thead{ \tiny P_R (l_1(X_{.\widehat{S}_{G}}) \\ = l_1(X_{.\widehat{S}_{}}))}$ & $\thead{ \tiny \widehat P(\bar{l}_1(X_{.\widehat{S}_{G}}) \\> \bar{l}_1(X_{.\widehat{S}_{}}))}$ \\
\midrule
Iris & 4 & 75 & 25 & 1 & 13.8 ± 7.3 & 7 ± 1 & 6.9 ± 1.4 & 0.96 & 0. & 2.4e-05 \\
Wine & 6 & 89 & 25 & 1 & 7.7 ± 0.3 & 13 ± 2 & 7.6 ± 0.3 & 0.64 & 0.16 & 6.3e-04 \\
Ethanol & 2 & 756 & 100 & 1 & 2.6 ± 0.3 & 90 ± 165 & 2.5 ± 0.2 & 0.66 & 0.17 & 2.1e-05 \\
\bottomrule
\end{tabular}
\caption{Experimental parameters and results.
For Iris and Wine, $P$ results from random downsampling by a factor of $2$ to create $R$ replicates.
$P_R$ values are empirical probabilities, while estimated P-values $\widehat P$ are computed by paired two-sample T-test on  $l_1(X_{.\widehat S})$ and $l_1(X_{.\widehat S_{G}})$.
For brevity, in this table $\widehat S := \widehat {S}_{TSIP}$.
}
\label{tab:experiments}
\end{table}

\begin{figure}[t]
    \centering
    % Subfigure for Wine dataset
    \begin{subfigure}[b]{0.3\textwidth}
        \centering
        \includegraphics[width=\textwidth]{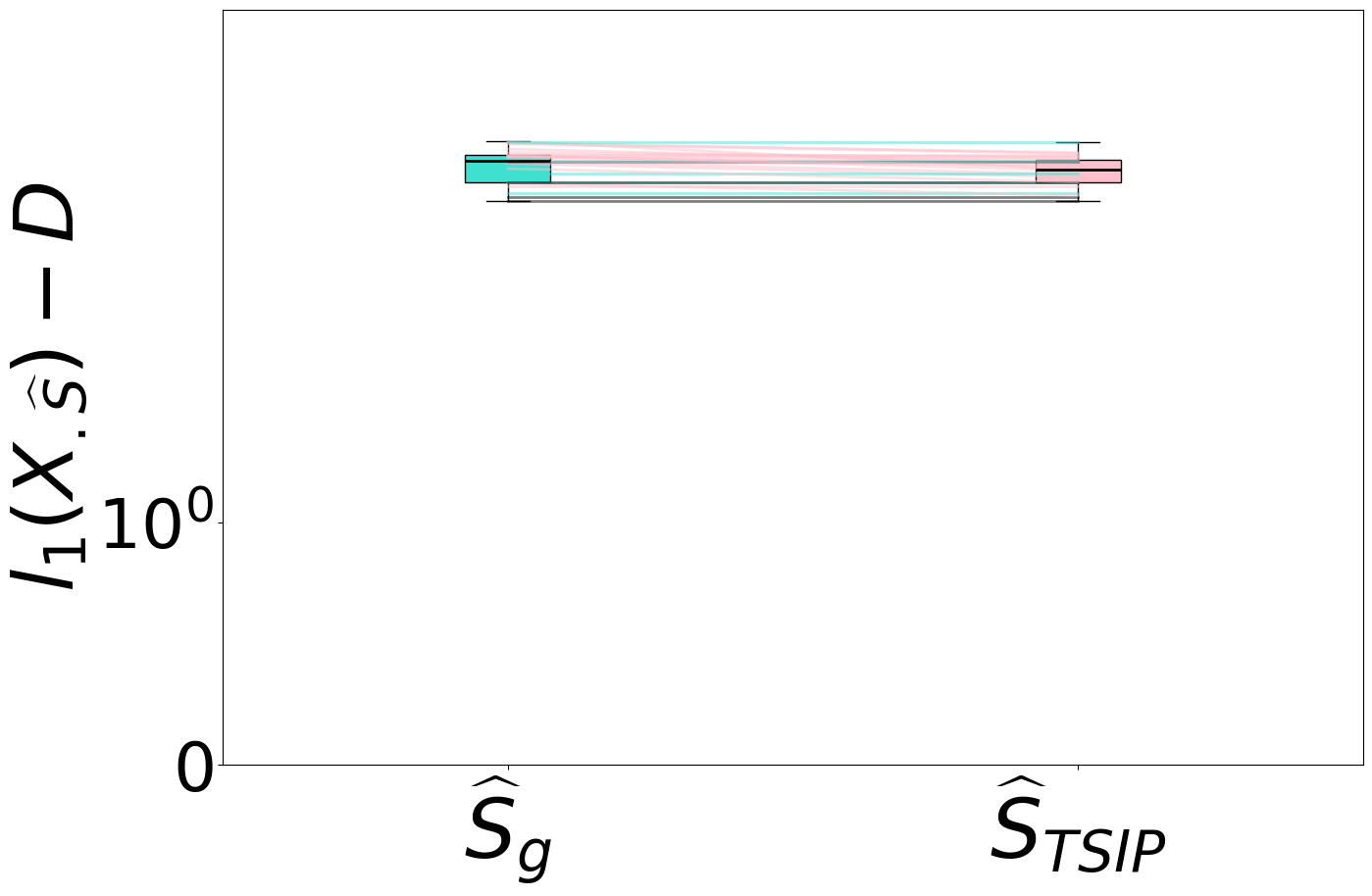}
        \caption{Wine dataset}
        \label{fig:wine_isometry_losses}
    \end{subfigure}
    \hfill
    % Subfigure for Iris dataset
    \begin{subfigure}[b]{0.3\textwidth}
        \centering
        \includegraphics[width=\textwidth]{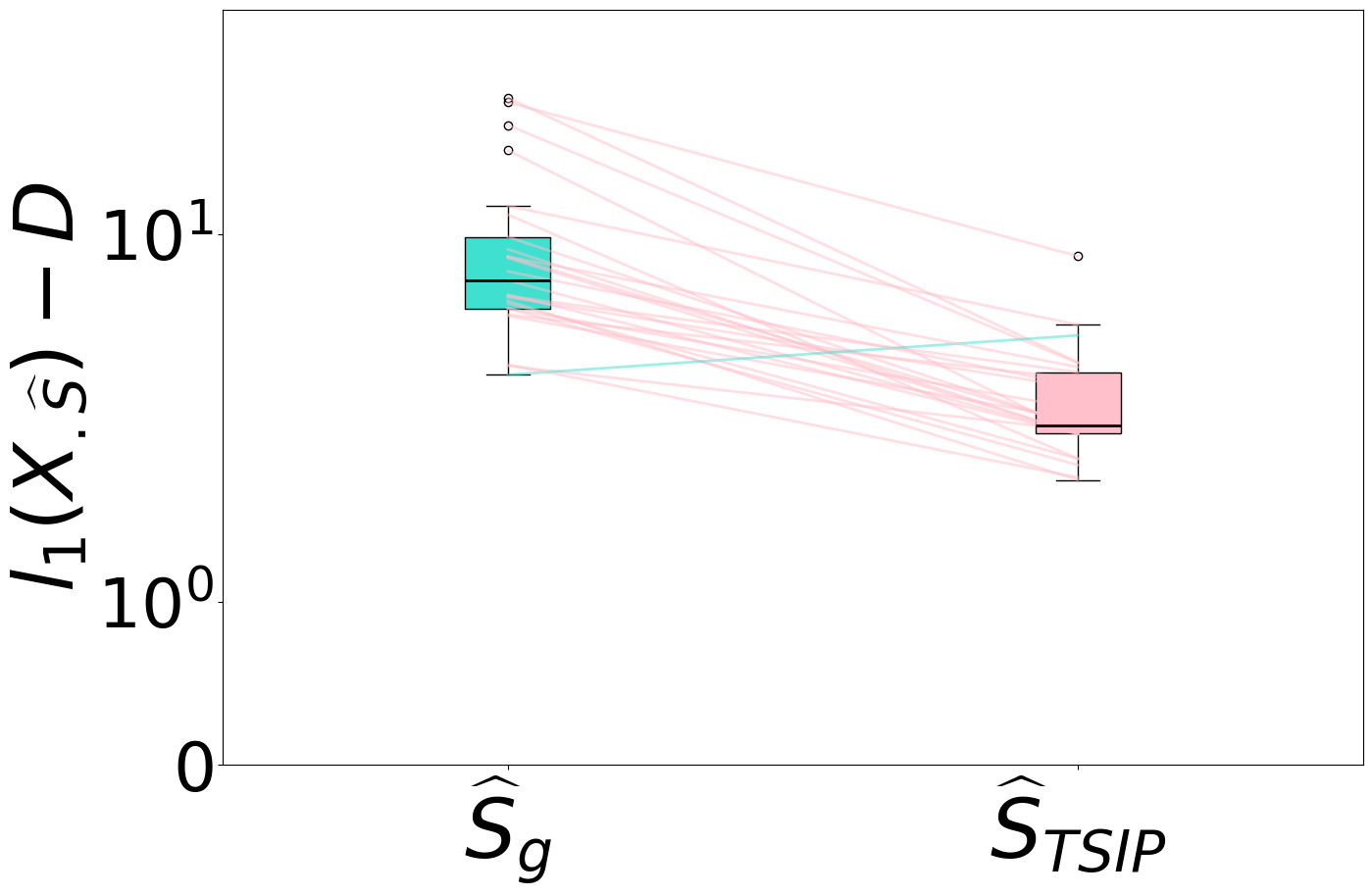}
        \caption{Iris dataset}
        \label{fig:iris_isometry_losses}
    \end{subfigure}
    \hfill
    % Subfigure for Ethanol dataset
    \begin{subfigure}[b]{0.3\textwidth}
        \centering
        \includegraphics[width=\textwidth]{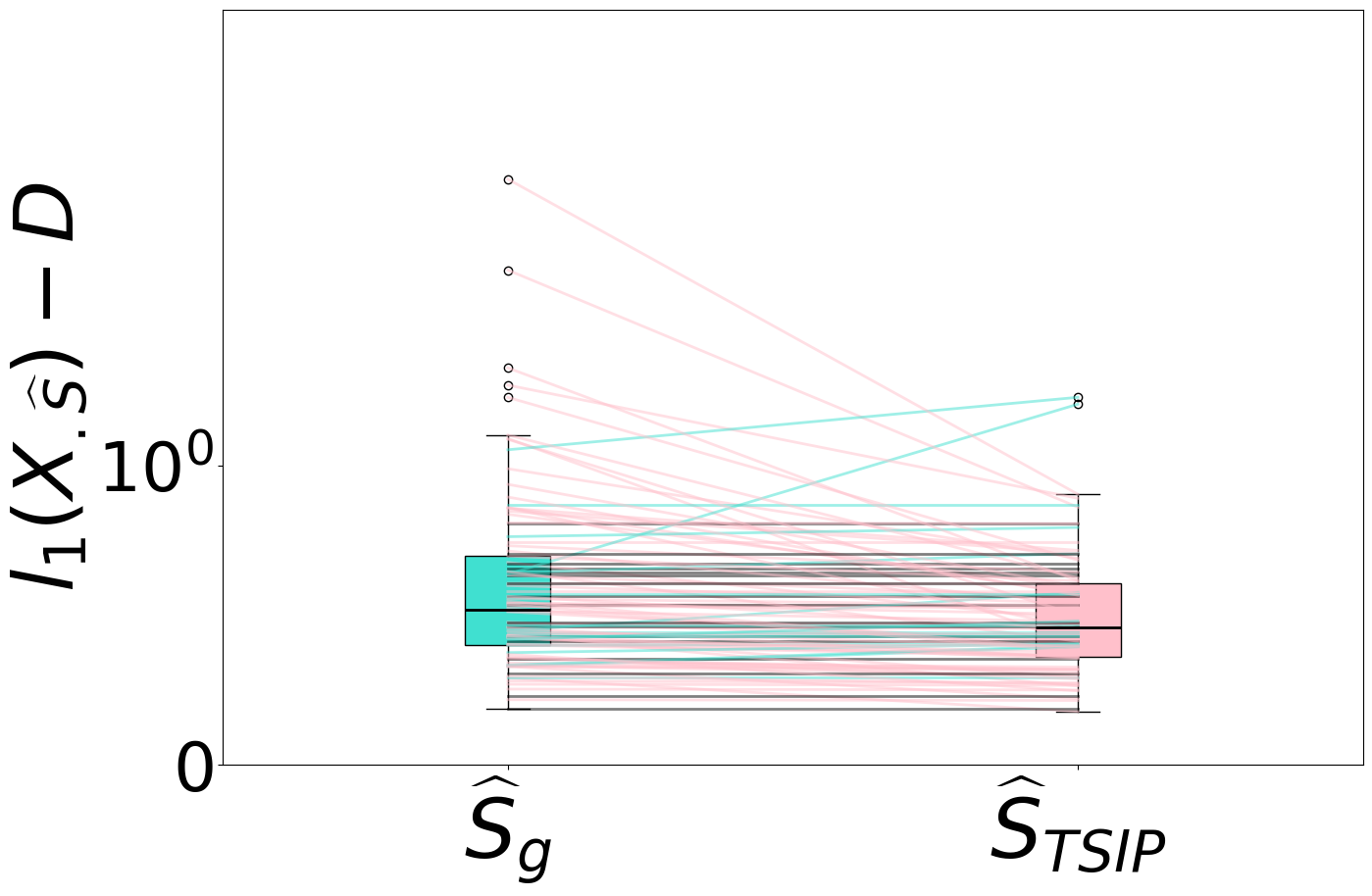}
        \caption{Ethanol dataset}
        \label{fig:ethanol_isometry_losses}
    \end{subfigure}
    \caption{Isometry losses $l_1$  for Wine, Iris, and Ethanol datasets across $R$ replicates.
    Lower brute losses are shown with turquoise, while lower two stage losses are shown with pink.
    Equal losses are shown with black lines.
    As detailed in Table \ref{tab:experiments}, losses are generally lower for two-stage isometry pursuit solutions.}
    \label{fig:isometry_losses}
\end{figure}

\section{Discussion}
\label{sec:discussion}

We have shown that multitask basis pursuit can help select isometric submatrices from appropriately normalized wide matrices.
This approach - isometry pursuit - is a convex alternative to greedy methods for selection of orthonormalized features from within a dictionary.
Isometry pursuit can be applied to diversification and geometrically-faithful coordinate estimation.
Our experiments exemplify these applications, but more can be done.
One potential application is diversification in recommendation systems \citep{Carbonell2017-gi, Wu2019-uk, Langchain} and other retrieval systems such as in RAG \citep{Gao2023-cn, Pickett2024-ad, In2024-um, Weiss2024-xm, Vectara}.
Another is decomposing interpretable yet overcomplete dictionaries in transformer residual streams, with each token considered as generating its own tangent space \citep{templeton2024scaling, Makelov2024-bw}.

Compared with the greedy algorithms used in such areas \citep{Carbonell1998-ji, Barioni, Drosou, Qin2012-ok, KUNAVER2017154, Guo-shengbo, Abdool,Yu2016AGA,  Huang2024-wr, Pickett2024-ad}, the convex reformulation may add speed and convergence to a global minima.
The comparison of greedy \cite{Mallat93-wi, Mallat, Pati-93, Tropp05-ml} and convex \citep{Chen2001-hh, Tropp06-sg,Chen2006TheoreticalRO} basis pursuit formulations has a rich history, and theoretical understanding of the behavior of this approximation is evolving.
Diversification problems have been cited as NP-hard, and isometry pursuit can be considered analogous to them in the sense of basis pursuit and the lasso against best subset selection, with the caveat that best subset selection of the basis pursuit loss minimizer isn't totally equivalent to isometry pursuit even though they share the same unique optimum.
Characterization of solutions resulting from removal of the restriction $P = D$ on the conditions of Proposition \ref{prop:unitary_selection} may help justify the second selection step.
That the solution of a lasso problem can sometimes be a non-singleton set is well-known \citep{Osborne2000OnTL, DOSSAL2012117, Chrtien2011OnTG, Tibshirani2012TheLP, Ewald2017OnTD, Ali2018TheGL, Schneider2020-qt, Mishkin2022TheSP,Dupuis2019TheGO,Debarre2020OnTU,Everink2024TheGA}.
Perhaps surprisingly, it appears empirically that for isometry pursuit that this can occur even when the design matrix is not in general position.

This convex set appears to contain the sparsest solution. 
The convergence of SCS algorithm to the 2-norm minimizing solution due to the Lagrangian dual constraint penalty and the convexity of the loss minimizer preimage suggest that a related two stage procedure always succeeds in identifying the brute $\|\|_{1,2}$ minimizer.
Related conditions have been discussed in \citet{Donoho2006ForML, Mishkin2022TheSP}, and we examine this topic experimentally in Section \ref{sec:deep_dive}.

Algorithmic variants include the multitask lasso \citep{ Hastie2015-qa} extension of our estimator, as well as characterization of $D$ function selection within $\mathbb R^B$.
Tangent-space specific variants have been studied in more detail in \citet{Koelle2022-ju, Koelle2024-no} with additional grouping across datapoints, and a corresponding variant of the isometry theorem that missed non-uniqueness was claimed in \citet{Koelle2022-lp}.
Comparison of our loss with curvature - whose presence prohibits $D$ element isometry - could prove fertile, as could comparison with the so-called restricted isometry property used to show guaranteed recovery at fast convergence rates in supervised learning \citep{Candes2005-dd, Hastie2015-qa}.

\begin{acknowledgments}
M.M. gratefully acknowledges the DataShape Group at INRIA Saclay and the Institute for Mathematical and Statistical Innovation (IMSI) for hospitality while a portion this research was carried out.
\end{acknowledgments}

\newpage

\bibliography{ref}

\newpage

\section{Supplement}

This section contains algorithms, proofs, and experiments in support of the main text.

\subsection{Algorithms}
\label{sec:algorithms}

We give definitions of the brute and greedy algorithms for the combinatorial problem studied in this paper.
The brute force algorithm is computationally intractable for all but the smallest problems, but always finds the global minima.

\begin{algorithm}[H]
\caption{\brute(Matrix ${X} \in \mathbb{R}^{D \times P}$, objective $f$)}
\begin{algorithmic}[1]
\FOR{each combination $S \subseteq \{1, 2, \dots, P\}$ with $|S| = D$}
    \STATE Evaluate $f({X}_{.S})$
\ENDFOR
\STATE {\bf Output} the combination $S^*$ that minimizes $f({X}_{.S})$
\end{algorithmic}
\end{algorithm}

Greedy algorithms are computationally expedient but can get stuck in local optima \citep{Cormen, Russell-09}, even with randomized restarts \citep{Dick2014HowMR}.

\begin{algorithm}[H]
\caption{\greedy(Matrix ${X} \in \mathbb{R}^{D \times P}$, objective $f$, selected set $S = \emptyset$, current size $d=0$)}
\begin{algorithmic}[1]
\IF{$d = D$}
    \STATE {\bf Return} $S$
\ELSE
    \STATE {\bf Initialize} $S_{\text{best}} = S$
    \STATE {\bf Initialize} $f_{\text{best}} = \infty$
    \FOR{each $p \in \{1, 2, \dots, P\} \setminus S$}
        \STATE {\bf Evaluate} $f({X}_{.(S \cup \{p\})})$
        \IF{$f({X}_{.(S \cup \{p\})}) < f_{\text{best}}$}
            \STATE {\bf Update} $S_{\text{best}} = S \cup \{p\}$
            \STATE {\bf Update} $f_{\text{best}} = f(\mathcal{X}_{.(S \cup \{p\})})$
        \ENDIF
    \ENDFOR
    \STATE {\bf Return} \greedy(${X}$, $f$, $S_{\text{best}}$, $d+1$)
\ENDIF
\end{algorithmic}
\end{algorithm}

\newpage

\subsection{Proofs}
\label{sec:proofs}

\subsubsection{Proof of Proposition \ref{prop:basis_pursuit_selection_invariance}}
\label{proof:basis_pursuit_program_invariance}

In this proof we first show that the penalty $\|\beta\|_{1,2}$ is unchanged by unitary transformation of $\beta$.

 \begin{proposition}
 \label{prop:basis_pursuit_loss_equivalence}
 Let $U \in \mathbb R^{D \times D}$ be unitary.
 Then $\|\beta\|_{1,2} = \|\beta U \|$.
\end{proposition}

\begin{proof}
\begin{align}
\|\beta U \|_{1,2} &= \sum_{p = 1}^P \| \beta_{p.} U \| \\
&= \sum_{p = 1}^P \| \beta_{p.} \| \\
&= \|\beta \|_{1,2}
\end{align}
\end{proof}

We then show that this implies that the resultant loss is unchanged by unitary transformation of $ X$.

\begin{proposition}
 \label{prop:basis_pursuit_loss_equivalence}
 Let $U \in \mathbb R^{D \times D}$ be unitary.
 Then $\widehat \beta  (U  X) = \widehat \beta  (  X) U$.
\end{proposition}

\begin{proof}
\begin{align}
\widehat \beta  (U  X)  &= \arg \min_{\beta \in \mathbb R^{P \times D}} \|\beta\|_{1,2}  \; : \; I_{D} = U X \beta \\
&= \arg \min_{\beta \in \mathbb R^{P \times D}} \|\beta\|_{1,2}  \; : \; U^{-1} U = U^{-1} U X \beta U \\
&= \arg \min_{\beta \in \mathbb R^{P \times D}} \|\beta\|_{1,2}  \; : \;  I_D = X \beta U \\
&= \arg \min_{\beta \in \mathbb R^{P \times D}} \|\beta U \|_{1,2}  \; : \;  I_D = X \beta U \\
&= \arg \min_{\beta \in \mathbb R^{P \times D}} \|\beta \|_{1,2}  \; : \;  I_D = X \beta.
\end{align}
\end{proof}

\subsubsection{Proof of Proposition \ref{prop:unitary_selection}}
\label{sec:local_isometry_proof}

 \begin{proposition}
\label{prop:generalized_unitary_selection}
Let $w_c$ be a normalization satisfying the conditions in Definition \ref{def:symmetric_normalization}.
Then $\arg \min_{X_{.S} \in \mathbb R^{D \times D}} \widehat \beta_c ( X_{.S}) $ is orthonormal and, given $X$ is orthonormal, $ \| \beta \|_{1,2} \; : \; I_D = w ({  X}, c) \beta = D$.
 \end{proposition}
 
 \begin{proof}

The value of $D$ is clearly obtained by $\beta$ orthonormal, since by Proposition \ref{prop:basis_pursuit_selection_invariance}, for $X$ orthogonal, without loss of generality 
\begin{align}
\beta_{dd'} = \begin{cases} 1 & d = d' \in \{ 1 \dots D\}  \\
0 & \text{otherwise}
\end{cases}.
\end{align}
Thus, we need to show that this is a lower bound on the obtained loss.

From the conditions in Definition \ref{def:symmetric_normalization}, normalized matrices will consist of vectors of maximum length (i.e. $1$) if and only if the original matrix also consists of vectors of length $1$.
Such vectors will clearly result in lower basis pursuit loss, since longer vectors in $X$ require smaller corresponding covectors in $\beta$ to equal the same result.

Therefore, it remains to show that $X$ consisting of orthogonal vectors of length $1$ have lower loss compared with $X$ consisting of non-orthogonal vectors.
Invertible matrices $X_{.S}$ admit QR decompositions $\tilde X_{.S} = QR$ where $Q$ and $R$ are orthonormal and upper-triangular matrices, respectively \citep{Anderson1992-fb}.
Denoting $Q$ to be composed of basis vectors $[e_1 \dots e_D]$, the matrix $R$ has form
\begin{align}
R = \begin{bmatrix}
\langle e_1, X_{.S_1} \rangle & \langle e_1,  X_{.S_2} \rangle  &\dots &  \langle e_1,  X_{.S_D} \rangle \\
0 & \langle e_d,  X_{.S_2} \rangle & \dots  &  \langle e_2,  X_{.S_D} \rangle\\
0 & 0 & \dots & \dots  \\
\dots & \dots & \dots & \dots \\
0 & 0 & \dots & \langle e_D, X_{.S_D} \rangle 
\end{bmatrix}.
\end{align}
Thus, $|R_{dd} | \leq \|X_{.{S_{d}}}\|_2$ for all $d \in [D]$, with equality obtained only by orthonormal matrices.
On the other hand, by Proposition \ref{prop:basis_pursuit_selection_invariance}, $l_c(X) = l_c(R)$ and so $\|\beta\|_{1,2} = \|R^{-1}\|_{1,2}$.
Since $R$ is upper triangular it has diagonal elements $\beta_{dd} = R_{dd}^{-1}$ and so $\|\beta_{d.}\| \geq \| X_{.{S_d}}\|^{-1} = 1$.
That is, the penalty accrued by a particular covector in $\beta$ is bounded from below by $1$ - the inverse of the length of the corresponding vector in $X_{.S}$ - with equality occurring only when $X_{.S}$ is orthonormal.
\end{proof}

\newpage

\subsection{Support cardinalities}
\label{sec:support_cardinalities}

Figure \ref{fig:support_cardinalities} plots the distribution of $|\widehat{S}_{IP}|$ from Table \ref{tab:experiments} in order to contextualize the reported means.
While typically $|\widehat{S}_{IP}| << P$, there are cases for Ethanol where this is not the case that drive up the means.

\begin{figure}[t]
    \centering
    % Subfigure for Wine dataset
    \begin{subfigure}[b]{0.3\textwidth}
        \centering
        \includegraphics[width=\textwidth]{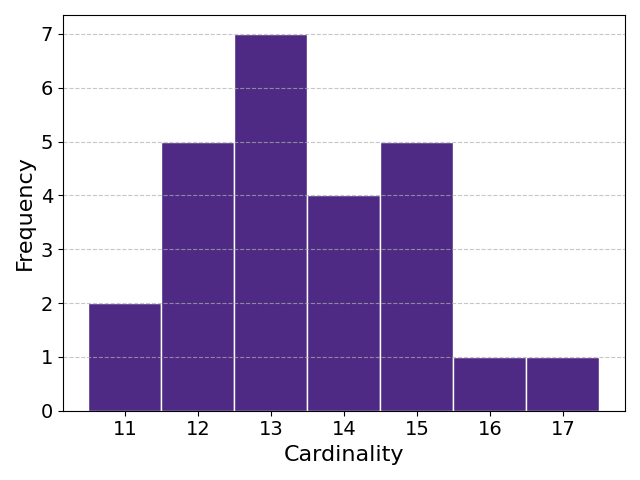}
        \caption{Wine Dataset}
        \label{fig:wine_cardinalities}
    \end{subfigure}
    \hfill
    % Subfigure for Iris dataset
    \begin{subfigure}[b]{0.3\textwidth}
        \centering
        \includegraphics[width=\textwidth]{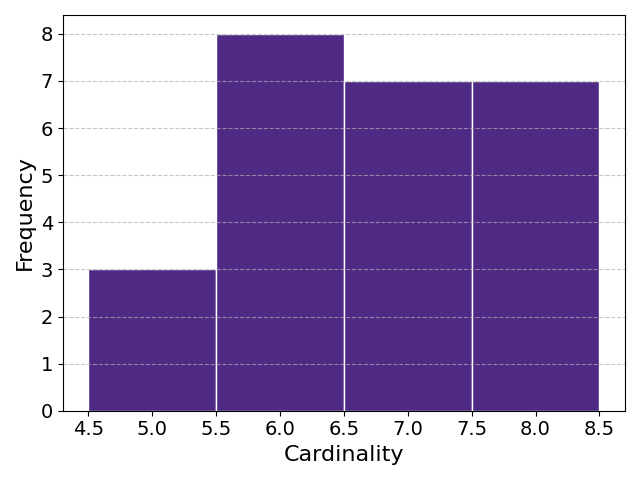}
        \caption{Iris Dataset}
        \label{fig:iris_cardinalities}
    \end{subfigure}
    \hfill
    % Subfigure for Ethanol dataset
    \begin{subfigure}[b]{0.3\textwidth}
        \centering
        \includegraphics[width=\textwidth]{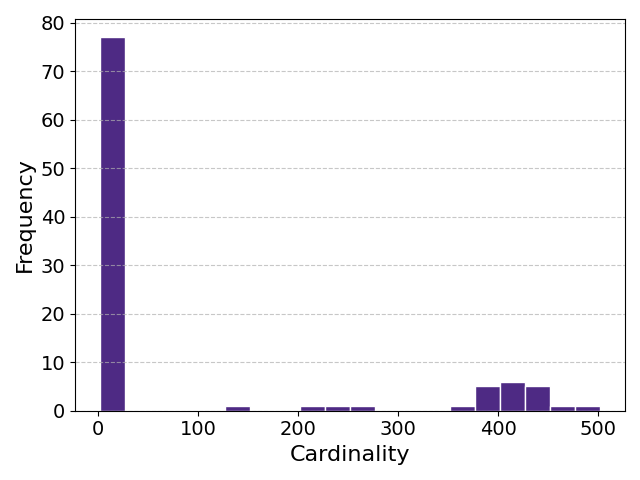}
        \caption{Ethanol Dataset}
        \label{fig:ethanol_cardinalities}
    \end{subfigure}
    \caption{Support Cardinalities for Wine, Iris, and Ethanol datasets}
    \label{fig:support_cardinalities}
\end{figure}

\newpage

\subsection{Proposition \ref{prop:unitary_selection} deep dive}
\label{sec:deep_dive}

As mentioned in Section \ref{sec:discussion}, the conditions under which the restriction $P=D$ in Proposition \ref{prop:unitary_selection} may be relaxed are of theoretical and practical interest.
The results in Section \ref{sec:experiments} show that there are circumstances in which the \greedy~ performs better than \tsip, so clearly \tsip~ does not always achieve a global optimum.
Figure \ref{fig:comparison_losses} gives results on the line of inquiry about why this is the case based on the reasoning presented in Section \ref{sec:discussion}.
In these results a two-stage algorithm achieves the global optimum of a slightly different brute problem, namely brute optimization of the multitask basis pursuit penalty $\|\cdot \|_{1,2}$.
That is, brute search on $\|\cdot \|_{1,2}$ gives the same result as the two stage algorithm with brute search on $\|\cdot \|_{1,2}$ subsequent to isometry pursuit.
This suggests that failure to select the global optimum by \tsip~ is in fact only due to the mismatch between global optimums of brute optimization of the multitask penalty and the isometry loss given certain data.
Theoretical formalization, as well as investigation of what data configurations this equivalence holds for, is a logical follow-up.

\begin{figure}[t] % Place at the top of the page
    \centering
    % Top-left plot
    \begin{subfigure}[b]{0.45\textwidth}
        \centering
        \includegraphics[width=\textwidth]{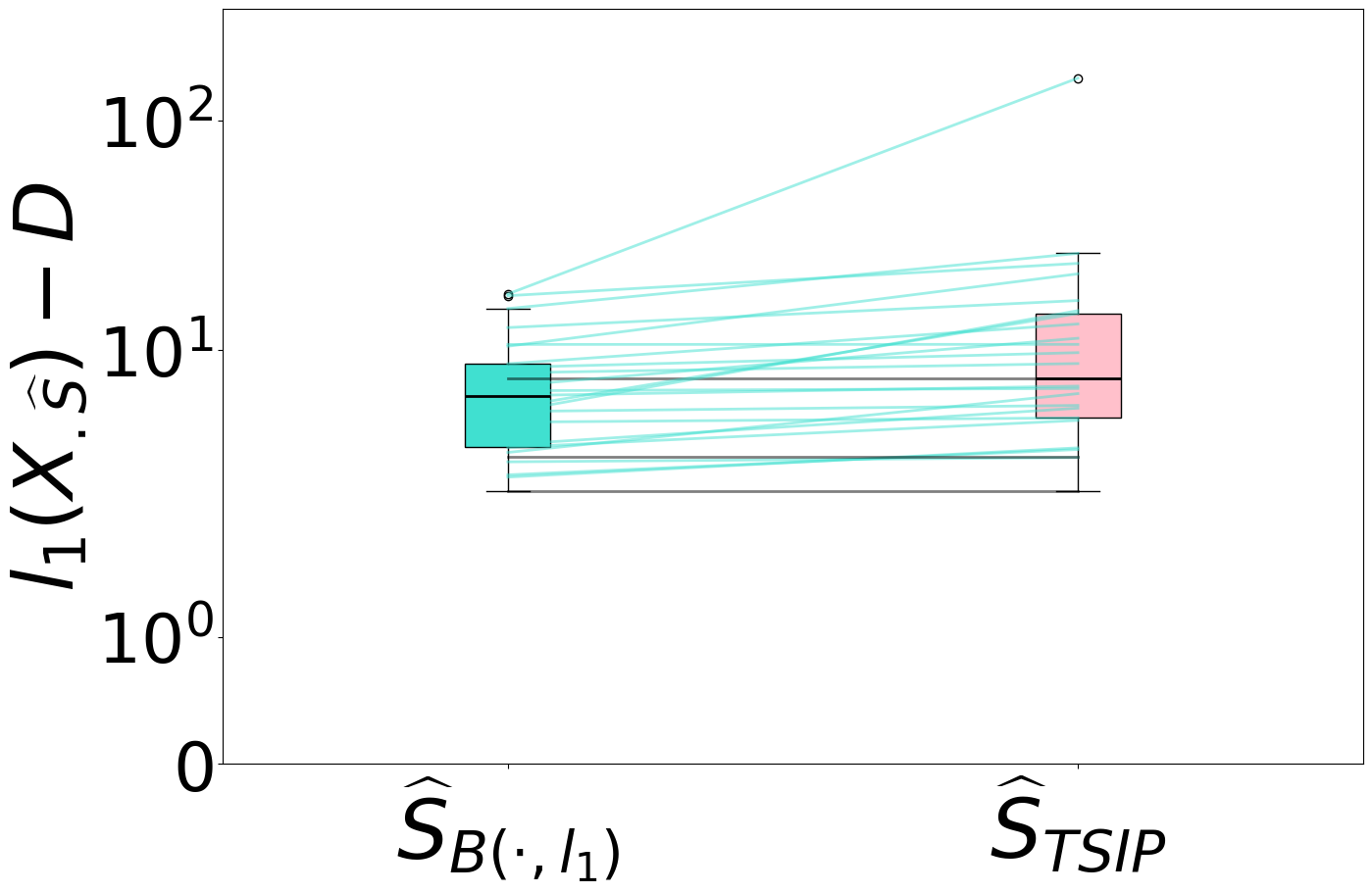}
        \caption{Iris Isometry Losses}
        \label{fig:iris_isometry_losses}
    \end{subfigure}
    \hfill
    % Top-right plot
    \begin{subfigure}[b]{0.45\textwidth}
        \centering
        \includegraphics[width=\textwidth]{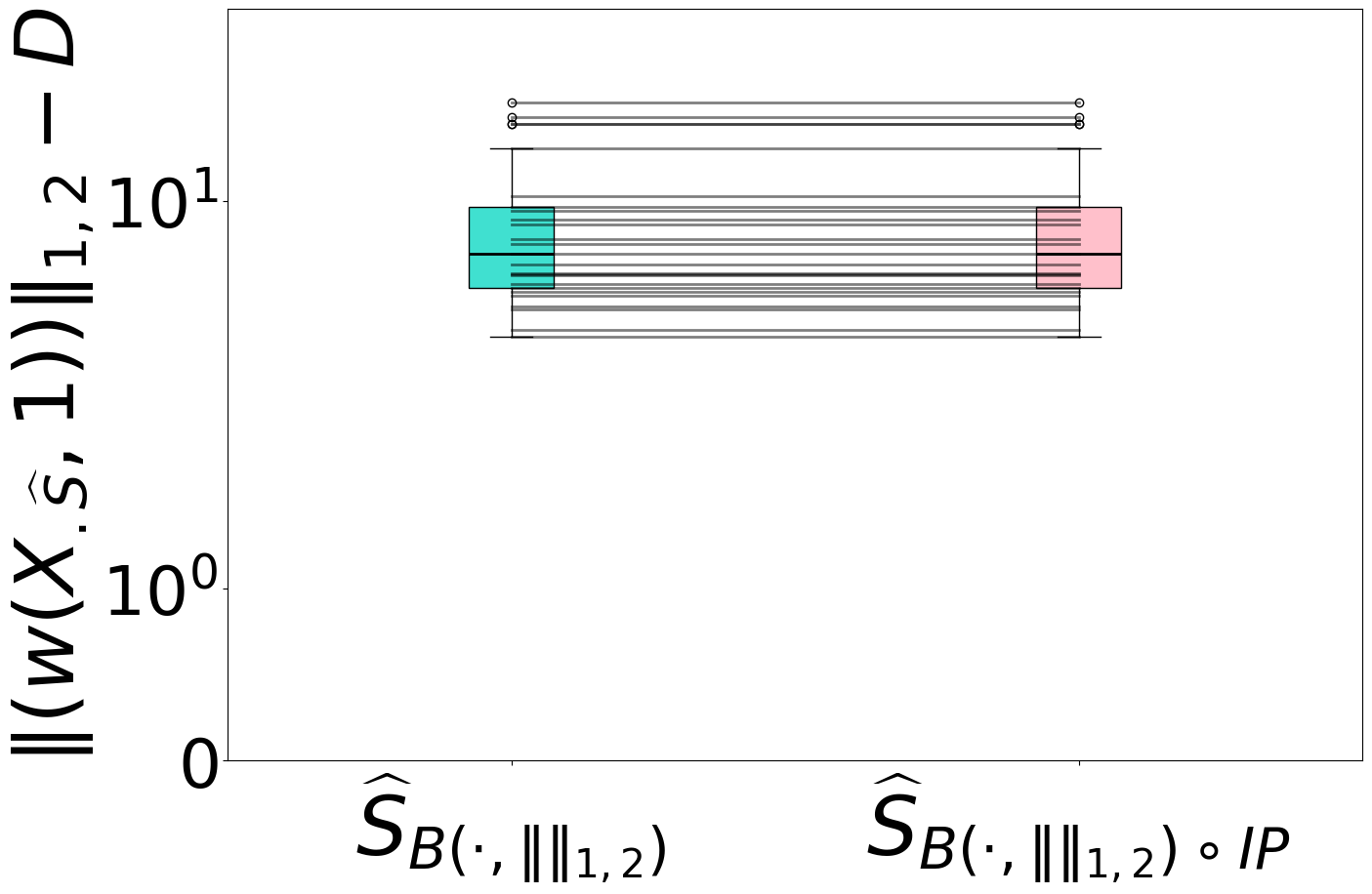}
        \caption{Iris Multitask Losses}
        \label{fig:iris_group_lasso_losses}
    \end{subfigure}

    \vspace{0.5cm} % Add vertical spacing between rows

    % Bottom-left plot
    \begin{subfigure}[b]{0.45\textwidth}
        \centering
        \includegraphics[width=\textwidth]{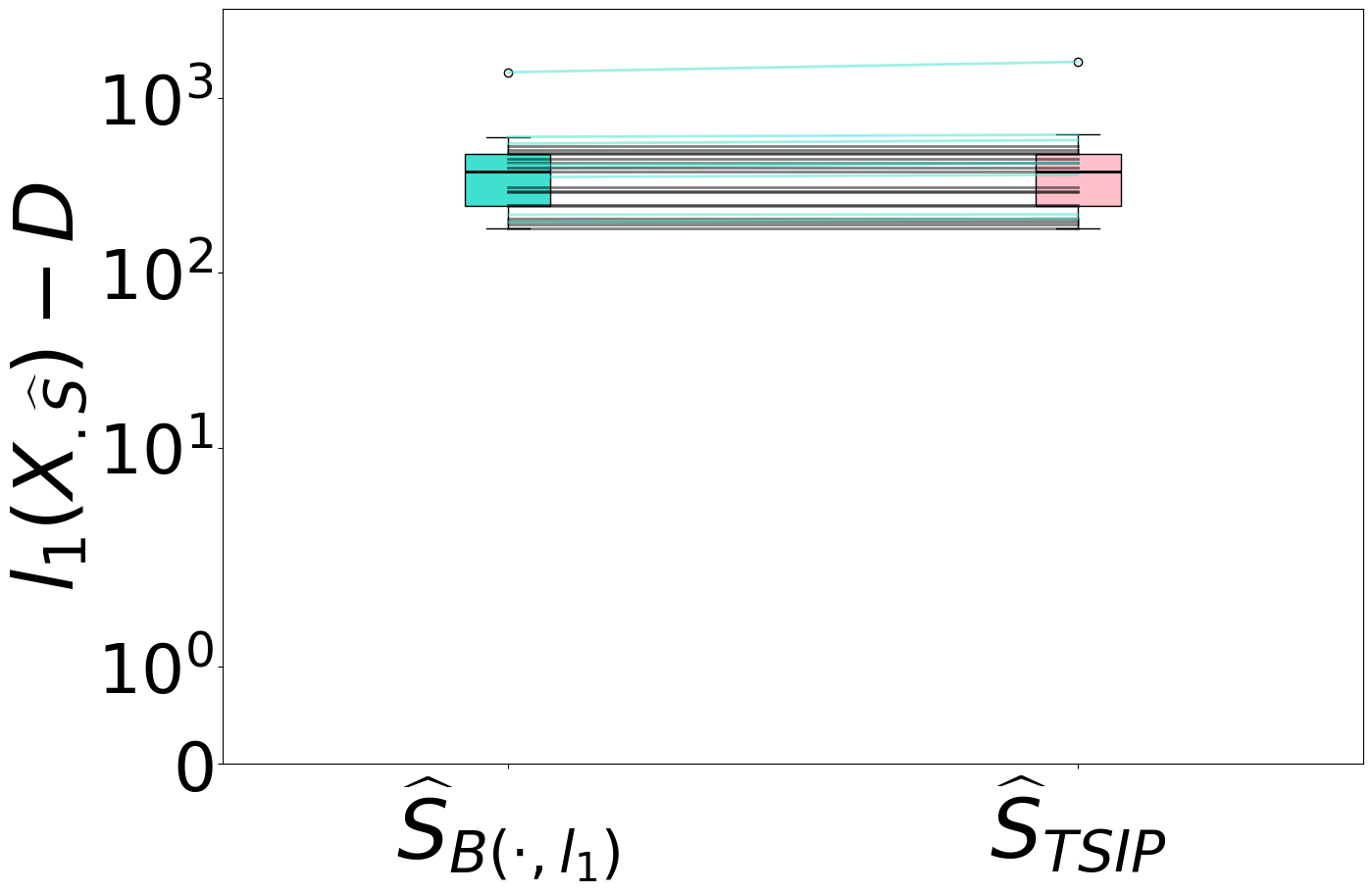}
        \caption{Wine Isometry Losses}
        \label{fig:wine_isometry_losses}
    \end{subfigure}
    \hfill
    % Bottom-right plot
    \begin{subfigure}[b]{0.45\textwidth}
        \centering
        \includegraphics[width=\textwidth]{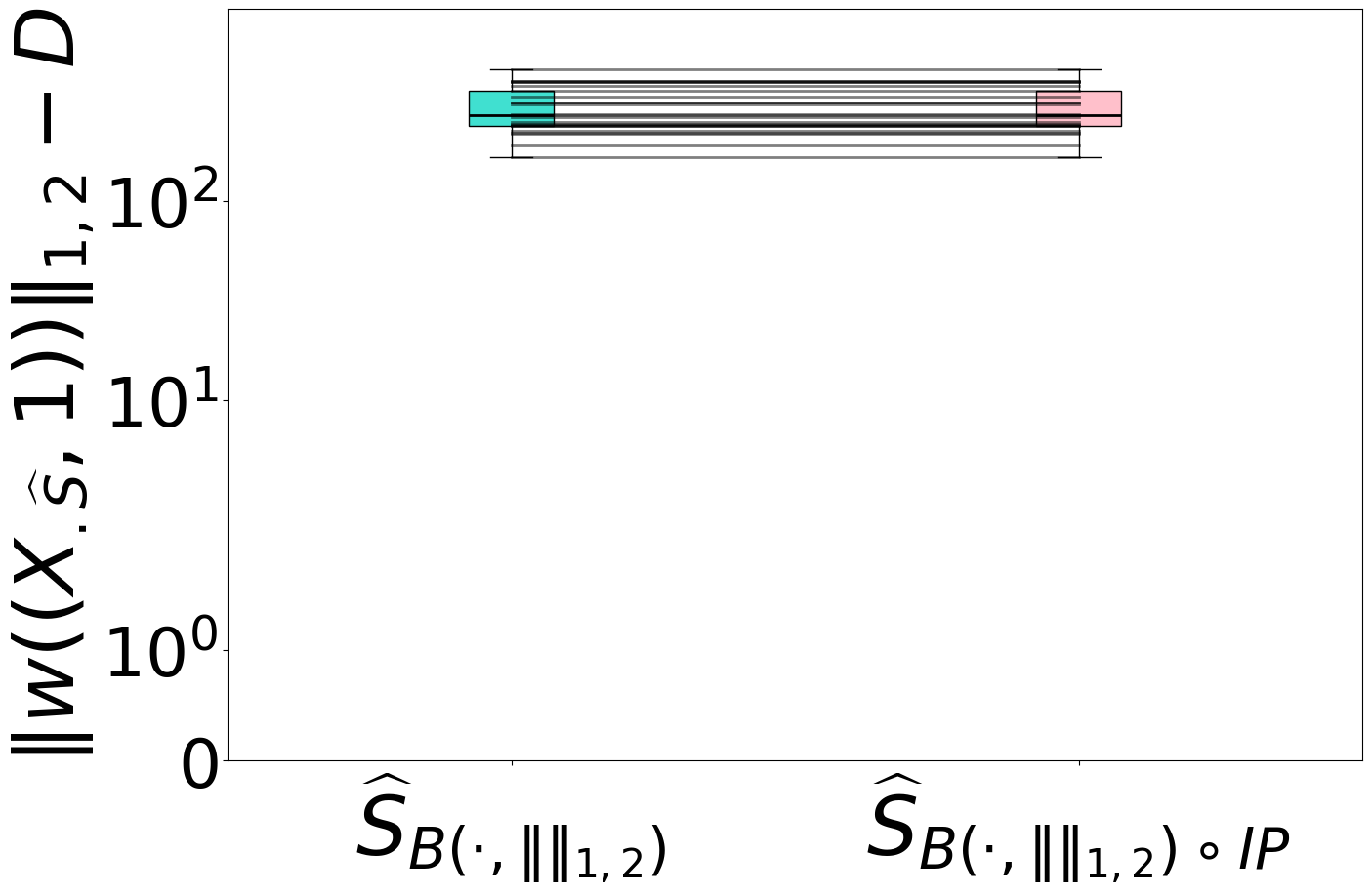}
        \caption{Wine Multitask Losses}
        \label{fig:wine_group_lasso_losses}
    \end{subfigure}
    \caption{Comparison of Isometry and Group Lasso Losses across $25$ replicates for randomly downsampled Iris and Wine Datasets with $(P,D) = (4,15)$ and $(13, 18)$, respectively.
    Note that this further downsampling compared with Section \ref{sec:experiments} was necessary to compute global minimizers of \brute.
    Lower brute losses are shown with turquoise, while lower two stage losses are shown with pink.
    Equal losses are shown with black lines.}
    \label{fig:comparison_losses}
\end{figure}

\newpage

\subsection{Timing}
\label{sec:timing}

While wall-time of algorithms is a non-theoretical quantity that depends on implementation details, it provides valuable context for practitioners.
We therefore report the following runtimes on a 2021 Macbook Pro.
The particularly high variance for brute force search in the second step of \tsip~ is likely due to the large cardinalities reported in Figure \ref{fig:support_cardinalities}.

\begin{table}[H]
\centering
\begin{tabular}{|c|c|c|c|}
\toprule
Name & IP & 2nd stage brute & Greedy \\
\midrule
Iris & 1.24 ± 0.02 & 0.00 ± 0.00 & 0.02 ± 0.00 \\
Wine & 2.32 ± 0.17 & 0.13 ± 0.12 & 0.03 ± 0.00 \\
Ethanol & 8.38 ± 0.57 & 0.55 ± 1.08 & 0.07 ± 0.01 \\
\bottomrule
\end{tabular}
\caption{Algorithm runtimes in seconds across replicates.}
\end{table}

\end{document}